\def\year{2020}\relax
\documentclass[letterpaper]{article} 
\usepackage{aaai20}  
\usepackage{times}  
\usepackage{helvet} 
\usepackage{courier}  
\usepackage[hyphens]{url}  
\usepackage{graphicx} 
\usepackage{graphicx}  
\usepackage{changes}

\usepackage{algorithm}
\usepackage{algorithmicx}
\usepackage[noend]{algpseudocode}
\usepackage{amsmath,mathtools}
\usepackage{amssymb}
\usepackage{amsthm}

\newtheorem{theorem}{Theorem}
\usepackage{subcaption}
\algrenewcommand\algorithmicindent{1.0em}%

\title{Multi-Resolution A*}
\author{Wei Du, Fahad Islam, {\normalfont and} Maxim Likhachev \\
The Robotics Institute \\
Carnegie Mellon University \\
{\tt \{wdu2, fi, maxim\}@cs.cmu.edu} }
\begin{document}
\maketitle
\begin{abstract}
Heuristic search-based planning techniques are commonly used for motion planning on discretized spaces.
The performance of these algorithms is heavily affected by the resolution at which the search space is discretized.
Typically a fixed resolution is chosen for a given domain.
While a finer resolution allows for better maneuverability, it significantly increases the size of the state space, and hence demands more search efforts.
On the contrary, a coarser resolution gives a fast exploratory behavior but compromises on maneuverability and the completeness of the search.
To effectively leverage the advantages of both high and low resolution discretizations, we propose Multi-Resolution~A*~(MRA*) algorithm, that runs multiple weighted-A*(WA*) searches having different resolution levels simultaneously and combines the strengths of all of them.
In addition to these searches, MRA* uses one \emph{anchor} search to control expansions from these searches.
We show that MRA* is bounded suboptimal with respect to the anchor resolution search space and resolution complete.
We performed experiments on several motion planning domains including 2D, 3D grid planning and 7 DOF manipulation planning and compared our approach with several search-based and sampling-based baselines.
\end{abstract}
\section{Introduction}
Search-based planners are known to be sensitive to the size of state spaces.
The three main factors that determine the size of a state space are the state dimension, the resolution at which each dimension is discretized and the size of the environment or the map~\cite{DBLP:journals/access/ElbanhawiS14}.
The size of state spaces grow exponentially with increased dimension and polynomially with increased resolution.
Search-based planning methods discretize the configuration space into cells.
A cell is the smallest unit of this discrete space and represents a small volume of configuration space state that lies within it.
The resolution of the discretization determines the size of a cell.
A representative state within a cell, commonly its geometric center is picked to denote a vertex for that cell.
\begin{figure}[t]
    \begin{subfigure}[b]{0.49\columnwidth}
        \centering
        \fbox{\includegraphics[width=0.95\columnwidth]{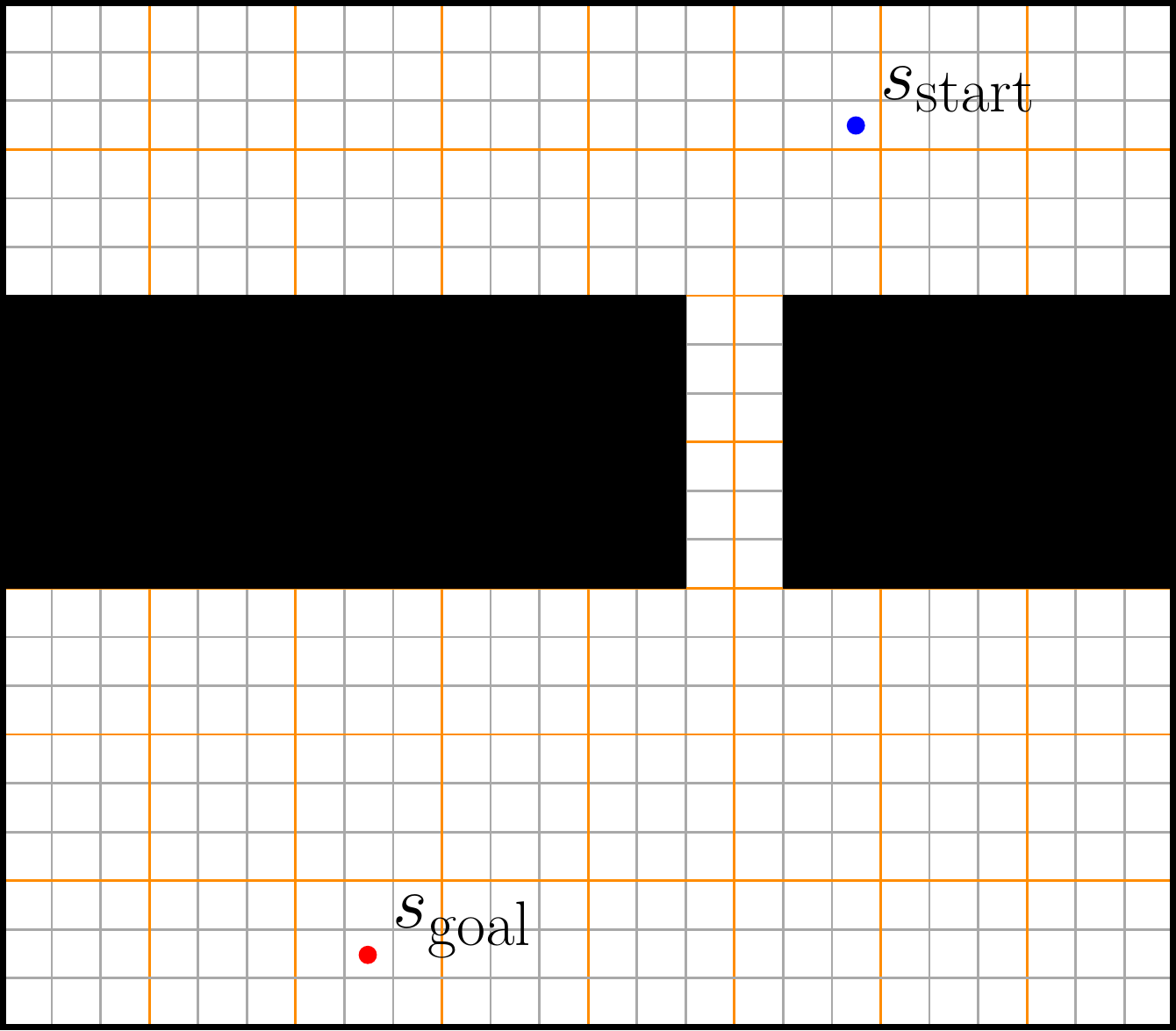}}\caption{Narrow passage}\label{fig:narrow_pasage}
    \end{subfigure}
    \begin{subfigure}[b]{0.49\columnwidth}
        \centering
        \fbox{\includegraphics[width=0.95\columnwidth]{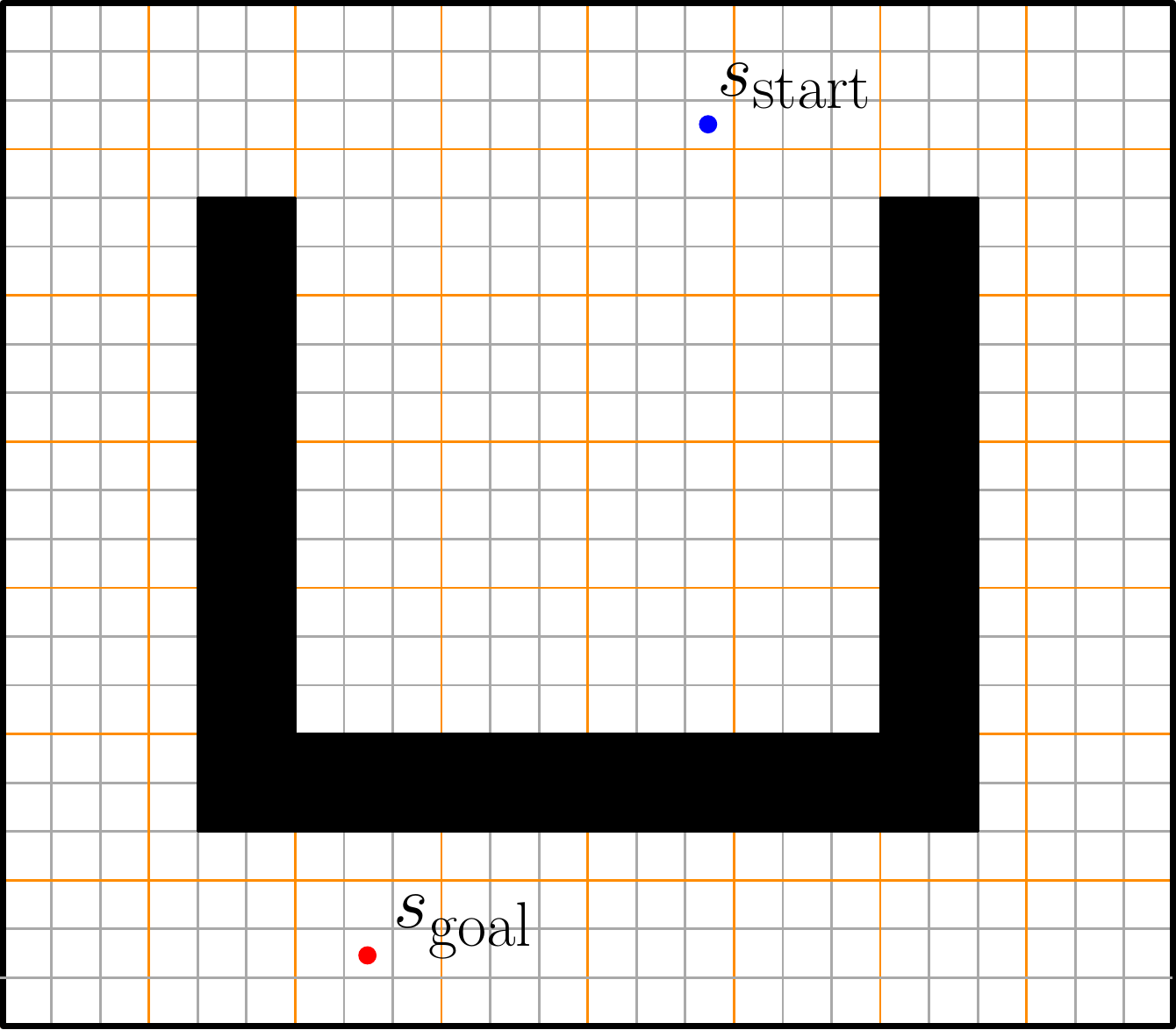}}\caption{Cul-de-sac}\label{fig:local_minima}
    \end{subfigure}
    \caption{Discretization of maps with high~(grey) and low~(orange) resolutions. 
            To the left, solution via transitions only through coarse cells does not exist.
            To the right, on the other hand, it is computationally expensive for the search to escape local minimum on a high resolution map.}
\label{fig:cartoonExp}
\end{figure}

Consider a large sized map most of which is free space, yet it has a number of narrow passages, which the planner has to find paths through for a point robot. Fig.~\ref{fig:cartoonExp} shows two snippets from this map discretized at two resolution levels.
For the example snippet shown in Fig.~\ref{fig:narrow_pasage}, to find a path from~$s_{\textrm{start}}$ to~$s_{\textrm{goal}}$, a search with the coarse resolution space will fail since the passage is too narrow for any of the coarse cells to be traversable. Not only does a low resolution space weaken the completeness guarantee, but it also sacrifices solution quality.

Consider another example map shown in Fig.~\ref{fig:local_minima}. For this problem instance, it is evident that the high resolution search would require a lot more expansions before it escapes the local minimum than the lower resolution search. 
Clearly, some portions of a map are best to be searched with coarse resolution while other portions may require a different, finer, resolution to find a solution.
To this end, we propose the Multi-Resolution A*~(that we shorten as MRA*) algorithm to combine the advantages of different resolution discretizations by employing multiple weighted-A*~(WA*)~\cite{pohl1973avoidance} searches that run on the different resolution state spaces simultaneously.

MRA* uses multiple priority queues that correspond to searches at each resolution level.
However, states from different discretizations that coincide are considered as the same state and thus, when generated by any search, they are shared between corresponding queues.
Our approach bears some resemblance to Multi-Heuristic A*~(MHA*) algorithm~\cite{doi:10.1177/0278364915594029};
MHA* uses multiple possibly inadmissible heuristics in addition to a single consistent anchor heuristic which is used to provide suboptimality bounds.
Instead of taking advantage of multiple heuristics in different searches, we leverage multiple state spaces at different resolutions.
To provide suboptimality guarantees we use an anchor search which runs on a particular resolution space.
We prove that MRA* is bounded suboptimal with respect to the optimal path cost in the anchor resolution space and resolution complete~\cite{lavalle2006planning}.

We conduct experiments on planning in~2D and~3D and on manipulation planning for a~7-DOF robotic arm, and compare MRA* with other search-based algorithms and sampling-based algorithms.
The results suggest that MRA* outperforms other algorithms for various performance metrics.

\section{Related Work}
Motion planning in high dimensional and large-scale domains is challenging both for search-based and sampling-based approaches~\cite{DBLP:journals/corr/abs-1806-07457}.

Sampling-based methods are popular candidates for high-dimensional motion planning problems. They have an advantage that they do not rely on discretizations, rather they use random sampling to discretize the state space.
Randomized methods such as RRT~\cite{lavalle2006planning} and RRT-Connect~\cite{DBLP:conf/icra/KuffnerL00} quickly explore high-dimensional space due to their random sampling feature.
Although fast, these algorithms are non-deterministic and provide no guarantees on the quality of solutions that they found.
Optimal variants such as RRT*~\cite{karaman2011sampling} provide asymptotic optimality guarantees, namely, they reach optimal solution as the number of samples grows to infinity.
Following RRT*, a family of algorithms including FMT*~\cite{DBLP:journals/ijrr/JansonSCP15}, RRT*-Smart~\cite{islam2012rrt} and Informed-RRT*~\cite{DBLP:conf/iros/GammellSB14} were developed to improve the convergence rate of RRT*.
These algorithms improve the quality of the solutions over time but do not provide bounds on the intermediate solution quality.
Moreover they often give inconsistent solutions - generate very different solutions for similar start and goal pairs - due to their inherent randomised behavior.

It is well-known that search-based planners suffer from the \textit{curse of dimensionality}~\cite{Bellman:1957}. They rely on a specific space discretization, the choice of which largely affects the computational complexity and properties of the algorithm. Several methods have been proposed to alleviate this problem on discrete grids.
Moore et al. came up with the Parti-game algorithm~\cite{DBLP:journals/ml/MooreA95}, which adaptively discretizes the map with high resolution at the border between obstacles and free space and low resolution on large free space.
Similarly, this notion is implemented via quad-tree search algorithms~\cite{DBLP:conf/icra/GarciaKB14,DBLP:conf/icra/YahjaSSB98}.
These algorithms are memory efficient in sparse environments, however, in cluttered environments, these approaches show little to no advantages over uniformly discretized map because of the overhead in book-keeping of the graph edges.
In our experiments, we show comparison with one of these adaptive discretization methods.

In addition to grid search, search over implicit graphs formulated as state lattices~\cite{DBLP:conf/iros/PivtoraikoK05} is ubiquitous in both navigation and planning for manipulation~\cite{DBLP:conf/icra/CohenCL10}.
These methods rely on motion primitives which are short kinematically feasible motions that the robot can execute.
In~\cite{DBLP:journals/ijrr/LikhachevF09}, graph search for autonomous vehicles was run on a multi-resolution lattice state space.
More specifically, they used high resolution space close to the robot or goal region and a low resolution action space elsewhere.
Similarly, the Hierarchical Path-Finding A*~(HPA*) algorithm~\cite{botea2004near} pre-processes maps into different levels of abstractions.
Then the complete solution is constructed by concatenating segments of trajectories within a local cluster which belongs to higher level abstraction path.
This approach relies on the condition that there is a smooth transition between high and low resolution abstractions.
Besides, these hierarchical structures require large memory footprint for maintaining the different abstractions and have significant computational overhead for pre-processing.
Compared to HPA*, MRA* runs search over an implicitly constructed graph (generated on the fly during search) and therefore, it requires less memory and no precomputation overhead.

Another class of methods plan in non-uniform state dimension and action to reduce the size of search state spaces~\cite{DBLP:conf/icra/CohenSCL11,DBLP:journals/ijrr/CohenCL14}.
Cohen et al. observed that not all the joints of a manipulator need to be active throughout the search, for example the joints at the end-effector might only be required to move near the goal region.
By restricting the search dimension in this manner, they gain considerable speedups.
Though efficient, this approach could potentially sabotage the completeness of the search.
To overcome this limitation, planning with adaptive dimensionality~\cite{DBLP:conf/aips/GochevSL13,DBLP:conf/socs/VemulaMO16} allows searching in lower dimension most of the time and only requires searching in the high dimension when necessary.
On related lines,~\cite{DBLP:conf/icra/BrockK01} decomposes the original problem into several high-dimensional and low-dimensional sub-problems in a divide-and-conquer fashion.
Their method provides guarantees on completeness but not optimality.
Our approach is different from these methods in that our decomposition is based on multiple resolutions instead of multiple dimensions in a way that provides completeness and bounded suboptimality guarantees.

\section{Multi-Resolution A*}
In a nutshell, MRA* employs multiple WA* searches in different resolution spaces (high and low) simultaneously and shares the states that coincide on the respective discretizations.
To gain more benefit out of the algorithm, the resolutions should be selected such that more sharing is facilitated.
If no sharing is allowed at all, the algorithm would degenerate into several independent searches and the solution will be returned by any search that would satisfy the termination criterion first.
In addition to these searches, MRA* uses an anchor search which is an optimal A* search, to anchor the state expansions from these searches in order to provide bounds on the solution quality.
In the remainder of this section we formally describe our algorithm.
We will also discuss the theoretical properties of this algorithm.

\subsection{Problem Definition and Notations}
In the following $S$ denotes a discretized domain.
Given a start state~$s_{\rm start}$ and a goal state~$s_{\rm goal}$ , the planning problem is defined as finding a collision free path from~$s_{\rm start}$ to~$s_{\rm goal}$ in ~$S$.
The cost from~$s_{\rm start}$ to a state~$s$ is denoted as~$g(s)$, optimal cost to come is denoted by~$g^*(s)$ and~$bp(s)$ is a back-pointer which points to the best predecessor of~$s$ (if one exists). The function~$c(s,s')$ denotes non-negative edge cost between any pair of states in~$S$.
Throughout the algorithm the anchor search and its associated data structures are indexed by~$0$ whereas other searches are denoted with indices~$1$ through~$n$.

We have multiple action sets $\{A_0, A_1, \ldots, A_n\}$ corresponding to different resolution spaces,
where~$A_i$ is a set of actions for resolution~$i$.
$\textsc{Succs}(s,i)$ returns all successors of $s$ for resolution~$i$ generated using the action space~$A_i$. 
$\textsc{GetSpaceIndices}(s)$ returns a list of indices of all the spaces which the state $s$ coincides with.
Furthermore, we assume that we have access to a consistent heuristic function~$h(s)$. 
Each WA* search uses a priority queue~$\textsc{OPEN}_i$ with the priority function~$\textsc{Key}(s,i)$ and a list of expanded states~$\textsc{CLOSED}_i$.
In the priority function~$\textsc{Key}(s,i)$(Alg.\ref{alg:mra_exp} Line~\ref{alg:mra_exp:priority}), all WA* searches share the same weight~$\omega_1$.
Additionally, each queue has a function~$\textsc{OPEN}_i\textsc{.MinKey()}$ which returns the minimum~$\textsc{Key}$ value for the $i$th queue. It returns~$\infty$ if the queue is empty.

\begin{algorithm}[tb]
\footnotesize
    \caption{Multi-Resolution A*}\label{alg:mra_main}
    \begin{algorithmic}[1]
        \Procedure{Main}{}
            \State $g(s_{\rm start}) = 0$; $g(s_{\rm goal}) = \infty$ \label{alg:mra_main:init1}
            \State $bp(s_{\rm start}) = bp(s_{\rm goal}) = $ \textbf{null}\label{alg:mra_main:bp}
            \For{$i = 0, ..., n$}
                \State OPEN\textsubscript{$i$} $\longleftarrow \emptyset$
                \State CLOSED\textsubscript{$i$} $\longleftarrow \emptyset$
                \If{$i \in$ \Call{GetSpaceIndices}{$s_{\textrm{start}}$}}
                    \State Insert $s_{\rm start}$ in OPEN\textsubscript{$i$} with \Call{Key}{$s,i$}
                \EndIf
            \label{alg:mra_main:init2}
            \EndFor 
            \While{OPEN\textsubscript{$i$} $\neq$ $\emptyset$ for each $i \in \{0, ..., n\}$ }\label{alg:mra_main:ept}
                    \State $i\longleftarrow$ \Call{ChooseQueue}{\ }\label{alg:mra_main:spolicy}
                    \If{OPEN\textsubscript{$i$}.\textsc{MinKey()} $\leq$ $\omega_2$ * OPEN\textsubscript{$0$}.\textsc{MinKey()}}\label{alg:mra_main:anchorcondition}
                        \If {$g(s_{\rm goal}) \leq OPEN\textsubscript{$i$}.\textsc{MinKey()}$}
                            \label{alg:mra_main:goal_cond}
                            \State Return path pointed by $bp(g(s_{\rm goal}))$
                            \label{alg:mra_main:ret_other}
                        \Else
                            \State $s$ = OPEN\textsubscript{$i$}.Pop()
                            \State \Call{ExpandState}{$s, i$}
                            \State Insert $s$ into CLOSED\textsubscript{$i$}\label{alg:mra_main:closeNormal}
                        \EndIf
                        \Else
                            \If {$g(s_{\rm goal}) \leq \omega_2 *OPEN\textsubscript{$0$}.\textsc{MinKey()}$}
                                \label{alg:mra_main:goal_cond2}
                                \State Return path pointed by $bp(g(s_{\rm goal}))$
                                \label{alg:mra_main:ret_anchor}
                            \Else
                                \State $s$ = OPEN\textsubscript{$0$}.Pop()
                                \State \Call{ExpandState}{$s, 0$}
                                \State Insert $s$ into CLOSED\textsubscript{$0$}
                            \EndIf
                        \EndIf
            \EndWhile
        \EndProcedure
    \end{algorithmic}
\end{algorithm}

\begin{algorithm}[tb]
\footnotesize
    \caption{ExpandState}\label{alg:mra_exp}
    \begin{algorithmic}[1]
        \Procedure{Key}{$s,i$}\label{alg:mra_exp:priority}
            \If {$i = 0$}
                \State \Return $g(s) + h(s)$
            \Else
                \State \Return $g(s) + \omega_1 h(s)$
            \EndIf
        \EndProcedure
        \Procedure{ExpandState}{$s$, $i$}\label{alg:mra_exp:entrance}
            \ForAll{ $s'$ $\in$ \Call{Succs}{$s, i$}}\label{alg:mra_exp:succ}
                \If{$s'$ was never generated}
                    \State $g(s') = \infty$; $bp(s') = $ \textbf{null};
                \EndIf
                \If{$g(s') > g(s) + c(s,s')$}\label{alg:mra_exp:improve}
                    \State $g(s') = g(s) + c(s,s')$; $bp(s') = s$
                    \For{each $i \in$ \Call{GetSpaceIndices}{$s'$}} 
                    \label{alg:mra_exp:all_res_start}
                        \If{$s' \notin$ CLOSED\textsubscript{i} }
                            \label{alg:mra_exp:close_end}
                            \State Insert/Update $s'$ in OPEN\textsubscript{$i$} with \Call{Key}{$s',i$}
                        \EndIf
                    \EndFor
                    \label{alg:mra_exp:all_res_end}
                \EndIf
            \EndFor
        \EndProcedure
    \end{algorithmic}
\end{algorithm}
\begin{figure*}[tbh]
    \begin{subfigure}[t]{0.24\textwidth}
        \centering
        \fbox{\includegraphics[width=.95\columnwidth]{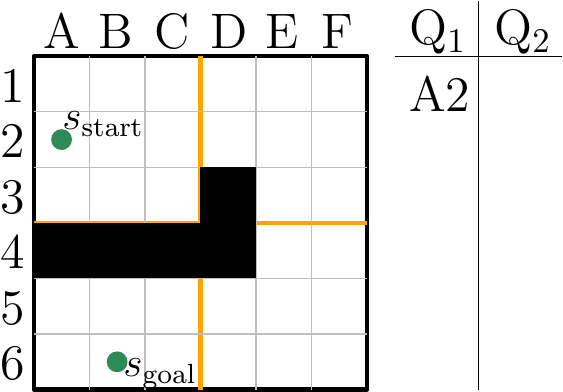}}%
        \caption{MRA* is initialized. State $s_{\rm start}$~(A2) is inserted into ${\rm Q}_1$.}%
        \label{fig:wkexp0}
    \end{subfigure}\hfill
    \begin{subfigure}[t]{0.24\textwidth}
        \centering
        \fbox{\includegraphics[width=.95\columnwidth]{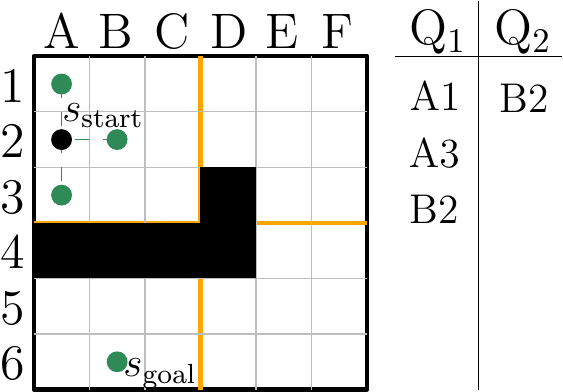}}%
        \caption{State A2 is expanded by high resolution search. Since state B2 lies at the center of a coarse cell, it is also inserted into ${\rm Q}_2$.}%
        \label{fig:wkexp1}
    \end{subfigure}\hfill
    \begin{subfigure}[t]{0.24\textwidth}
        \centering
        \fbox{\includegraphics[width=.95\columnwidth]{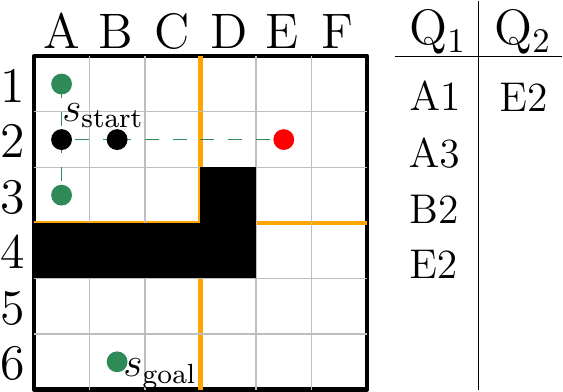}}%
        \caption{State B2 is expanded by low resolution search. The successor~E2 is inserted into both queues.}%
        \label{fig:wkexp2}
    \end{subfigure}\hfill
    \begin{subfigure}[t]{0.24\textwidth}
        \centering
        \fbox{\includegraphics[width=.95\columnwidth]{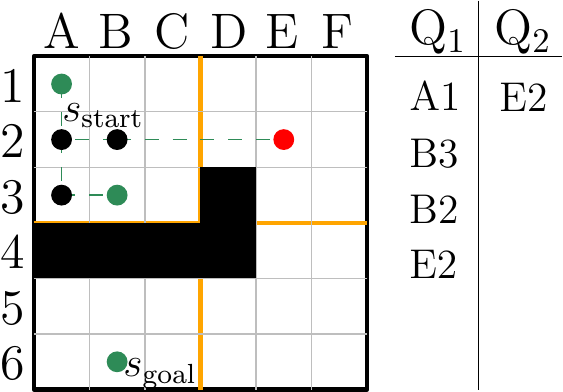}}%
        \caption{State A3 is expanded by high resolution search.}%
        \label{fig:wkexp3}
    \end{subfigure}\hfill
    \begin{subfigure}[t]{0.24\textwidth}
        \centering
        \fbox{\includegraphics[width=.95\columnwidth]{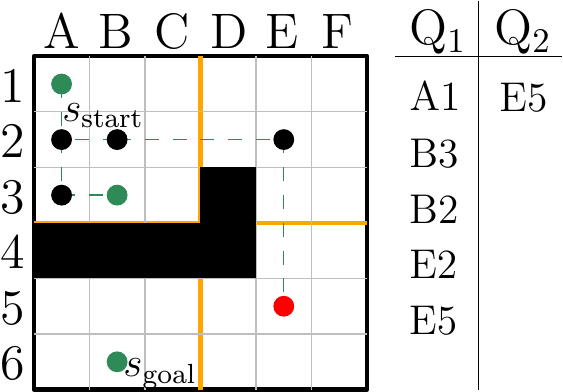}}%
        \caption{State E2 is expanded by low resolution search and the successor E5 is inserted into both queues.}%
        \label{fig:wkexp4}
    \end{subfigure}\hfill
    \begin{subfigure}[t]{0.24\textwidth}
        \centering
        \fbox{\includegraphics[width=.95\columnwidth]{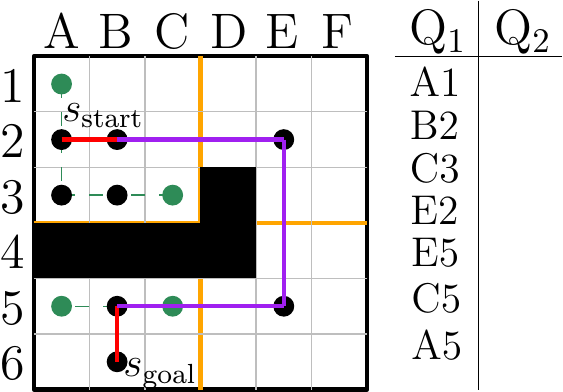}}%
        \caption{The last step when $s_{\rm goal}$ is expanded by high resolution search. A solution~(solid line segments) is found with~$8$ expansions in total.}%
        \label{fig:wkexp5}
    \end{subfigure}\hfill
    \begin{subfigure}[t]{0.24\textwidth}
        \centering
        \fbox{\includegraphics[width=0.93\columnwidth]{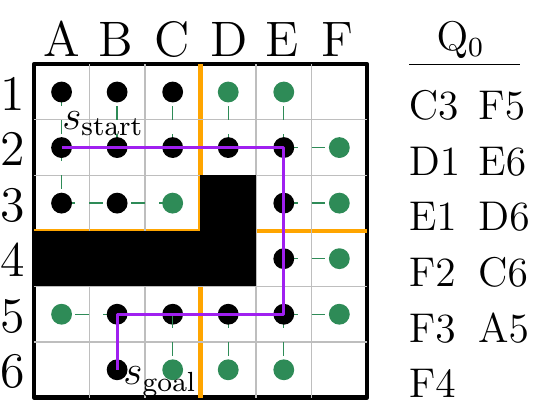}}%
        \caption{The final status of a high resolution search. The same solution~(purple) is found with~$17$ expansions.}%
        \label{fig:wkexp6}
    \end{subfigure}
    \caption{
    Illustration of MRA* algorithm---Thick (orange) lines and thin (grey) lines show the low and high resolution grids respectively.
    The heuristics used is Manhattan distance.
    MRA* initializes in Fig.~\ref{fig:wkexp0}. Figs.~\ref{fig:wkexp1} to ~\ref{fig:wkexp4} show the first four expansions of MRA* and Fig.~\ref{fig:wkexp5} shows the last expansion when the search terminates.
    OPEN lists for the high and low resolution searches are denoted as ${\rm Q}_1$ and ${\rm Q}_2$ respectively.
    Expanded states are shown in black, states in OPEN lists are shown in green and the states that coincide between the two spaces are shown in red.
    The path returned by MRA* is composed of edges from both high~(red) and low~(purple) resolution spaces.
    Fig.~\ref{fig:wkexp6} illustrates the behaviour of WA* search only in the high resolution grid.
    }
\label{fig:alg_example}
\end{figure*}

\subsection{Algorithm}
The main algorithm is presented in Alg.~\ref{alg:mra_main}.
The lines~\ref{alg:mra_main:init1}-~\ref{alg:mra_main:init2} initialize the $g$ values and back pointers of~$s_{\rm start}$ and~$s_{\rm goal}$, and OPEN and CLOSED for each queue and insert ~$s_{\rm start}$ into all queues with which~$s_{\rm start}$ coincides with the corresponding priority values.

The algorithm runs until all the priority queues get empty~(line~\ref{alg:mra_main:ept}) or any of the two termination criteria (lines~\ref{alg:mra_main:goal_cond} or~\ref{alg:mra_main:goal_cond2}) are met.
At line~\ref{alg:mra_main:spolicy}, in function~$\textsc{ChooseQueue()}$, we employ a scheduling policy to make decision on from which non-empty queue to expand a state in current iteration.
This scheduling policy could be a round-robin strategy, Dynamic Thompson Sampling~(DTS) policy or other scheduling policies, as is suggested in~\cite{DBLP:conf/ijcai/PhillipsNAL15}~\footnote{In DTS policy, the selection of a queue is viewed as a \textit{multi-arm bandit} problem~\cite{DBLP:conf/icmla/GuptaGA11}, where the reward from a "bandit"  is equal to the search progress made by the decision, reflected in the decrease of chosen queue's top state's heuristic value.}.
The condition in line~\ref{alg:mra_main:anchorcondition} controls the inadmissible expansions from other queues. 
Inadmissible searches are suspended and anchor search is employed whenever this condition fails.
As a consequence, the solution returned from any search will be within the suboptimality bound~$\omega_2$ of the optimal solution in the anchor space.
The expansions from the anchor queue monotonically increase~$\textsc{OPEN}_0\textsc{.MinKey()}(s)$ as the anchor is a pure A* search, allowing more states to be expanded from the other queues.
The minimum priority state is popped from OPEN\textsubscript{$i$} and then added to the corresponding CLOSED\textsubscript{$i$}.

Details of a state expansion are presented in~Alg.~\ref{alg:mra_exp}.
The \textsc{ExpandState($s,i$)} function  ``partially" expands state $s$ in the search $i$ by using actions~$A_i$.
If the successors of $s$ are duplicates of states in other spaces, they are inserted or updated in the corresponding searches as well.
This is how the paths or the $g$ values of the states are shared between the different searches.
In this procedure, the condition at Line~\ref{alg:mra_exp:improve} indicates that a state will only be updated in a queue if its $g$ value is improved.
A state ~$s'$ is only inserted in a queue if it was not expanded before from the same queue and if it coincides with the discretization of that queue~(see lines~\ref{alg:mra_exp:all_res_start}-~\ref{alg:mra_exp:all_res_end}).

Fig.~\ref{fig:alg_example} provides a simple 2D illustration of the MRA* algorithm.
We use two resolutions (high and low) in this example and MRA* alternatively expands states from the two queues.
The cell size~(the length of a side) of the low resolution space is 3 times the size of the high resolution space.
For the sake of simplicity, we assume that the suboptimality bound $\omega_2$ is very high such that anchor queue is never expanded i.e the condition in line~\ref{alg:mra_main:anchorcondition} is never violated. We also assume that the weights $\omega_1$ are high enough that the WA* searches are purely greedy.
Fig.~\ref{fig:wkexp6} shows the result if we would only run a single high-resolution search for the same example for comparison.
It is evident that benefiting from the sharing feature between multiple resolution spaces, MRA* found the solution with much less expansions than the high resolution WA* search.

\subsection{Analysis}
\begin{theorem}
    MRA* partially expands a state at most once with respect to each inadmissible search and anchor search.
\end{theorem}
\noindent
This holds true by construction (see lines~\ref{alg:mra_exp:all_res_start}-~\ref{alg:mra_exp:all_res_end})
\noindent

\begin{theorem}
MRA* is complete in the union space of all~$n~+~1$ resolution spaces.

\end{theorem}
\noindent
The union space is defined as the space constructed as a result of sharing coincident states between the different resolution spaces.
This theorem also holds by construction as the algorithm terminates only if it finds a solution or all the resolution spaces get exhausted~(Alg.~\ref{alg:mra_main}, line~\ref{alg:mra_main:ept})
\noindent

\begin{theorem}
    In MRA*, solution returned by any search $i$ with total cost $g_i(s_{\rm goal})$ is bounded as:
    $$
        g_i(s_{\rm goal}) \leq \omega_2 * g^*_0(s_{\rm goal})
    $$
    where $g^*_0(s_{\rm goal})$ is the optimal solution with respect to anchor resolution.
\end{theorem}
\begin{proof}
If the anchor search terminates at Alg.~\ref{alg:mra_main}, line~\ref{alg:mra_main:ret_anchor} then because the anchor search is an optimal A* search, from~\cite{pearl1984heuristics}, we have

\begin{equation} \label{eq1}
\begin{split}
g_0(s_{\rm goal}) \leq \omega_2 * g_0^*(s_{\rm goal})
\end{split}
\end{equation}

If any other search terminates (Alg.~\ref{alg:mra_main}, line~\ref{alg:mra_main:ret_other}), then from lines~\ref{alg:mra_main:anchorcondition} and~\ref{alg:mra_main:goal_cond}, and because the anchor search is A* search we have,
\begin{equation} \label{eq2}
\begin{split}
g_i(s_{\rm goal}) & \leq  \omega_2 * \text{OPEN}\textsubscript{$0$}.\textsc{MinKey()} \\
                  & \leq \omega_2 * g_0^*(s_{\rm goal}) \text{ From~\cite{pearl1984heuristics}}
\end{split}
\end{equation}
\end{proof}

\section{Experiments and Results}
We evaluate our algorithm on~2D,~3D and~7D domains and report comparisons with different search-based and sampling-based planning approaches in terms of planning time, solution cost, number of expanded states~(only for search-based algorithms) and success rates. All experiments were run on an Intel~i7-3770 CPU~(3.40 GHz) with 16GB RAM.
In all experiments, we set a timeout of~$120$ seconds.
For 2D and 3D spaces, we used 8-connected and 26-connected grids. For 7D experiments we used PR2 robot's single-arm and constructed the graph using a manipulation lattice~\cite{DBLP:conf/icra/CohenSCL11}.
The heuristics used for~2D and~3D domains are octile distance and euclidean distance respectively.
For manipulation problems, the heuristic was computed by running a backward 3D Dijkstra's from the end-effector's position at the 6-DoF goal pose.
We used Euclidean distance as cost function for 2D and 3D, and Manhattan distance in joint angles for 7D.
For all the domains, the anchor search of MRA* is set as the highest resolution space.
As the queue selection policy, we used round-robin policy for~2D and~3D, and DTS for the 7D domain.
For every domain, we plot statistics showing improvements of MRA* over baselines, where improvements are computed as the average metric values of baselines divided by that of MRA*'s~ (Fig.~\ref{fig:all_results}). For these plots we only report results for common success tests.
In addition, we also show tabulated results for all the metrics (Table~\ref{table:rlt23d}).
The code of MRA* algorithm will be available here\footnote{http://www.sbpl.net/Software}.
\begin{table*}[ht]
    \caption{2D and 3D planning results.}\label{table:rlt23d}
    \begin{subtable}{\textwidth}
        \caption{2D Planning Results~(Map1 \& Map2)} \label{table:rlt2d}
        \resizebox{\columnwidth}{!}{%
             \begin{tabular}{c|c|c|c|c|c||c|c|c|c|c}
                    \hline \hline
                     & \multicolumn{5}{|c||}{Map1} & \multicolumn{5}{|c}{Map2}\\
                    \hline
                    Algorithm  & MRA*   & WA-MR   & WA-High   & WA-Low & QDTree &  MRA*   & WA-MR   & WA-High   & WA-Low & QDTree\\
                    \hline 
                    Success Rate (\%)  & \textbf{100}  & \textbf{100}  & \textbf{100}  & 95.5  & \textbf{100} & \textbf{100}    & 98.99   & 98.99 & 94.95  & \textbf{100}\\
                    \hline                           
                    Mean Time (s)      & 0.61   & 5.72   & 5.62  & \textbf{0.09} & 0.15 & 4.14 & 18.23  & 17.73  & \textbf{0.22} & 0.44\\
                    \hline                           
                    Mean Cost ($m$)  & \textbf{324.71} & 325.76  & 326.32 & 326.71 & 341.49 & \textbf{377.91}  & 379.51  & 382.35 & 380.55  & 396.93\\
                    \hline \hline
            \end{tabular}}
    \end{subtable}%
    \hfill
    \begin{subtable}{\textwidth}
        \centering
        \caption{3D Planning Results~(Map1 \& Map2)} \label{table:rlt3d}
        \resizebox{\columnwidth}{!}{%
             \begin{tabular}{c|c|c|c|c||c|c|c|c}
                    \hline \hline
                     & \multicolumn{4}{|c||}{Map1} & \multicolumn{4}{|c}{Map2}\\
                    \hline
                    Algorithm  & MRA*   & WA-MR   & WA-High   & WA-Low  &  MRA*   & WA-MR   & WA-High   & WA-Low \\
                    \hline 
                    Success Rate (\%)  & 100  & 100  & 100  & 100 & 100   & 100 & 100  & 100\\
                    \hline                           
                    Mean Time (s)      & 3.12   & 19.01  & 18.88 & \textbf{0.06} & 4.16    & 24.16  & 13.71 & \textbf{0.07}\\
                    \hline                           
                    Mean Cost ($m$)  & 40.13  & 38.38  & \textbf{37.04}   & 40.20 & 32.35   & 30.45 & \textbf{28.83}  & 31.89\\
                    \hline \hline
            \end{tabular}}
    \end{subtable}
    \label{table:results}
\end{table*}
\subsection{2D Space Planning Results}
\begin{figure}[tb]
    \centering
    \includegraphics[width=.95\columnwidth]{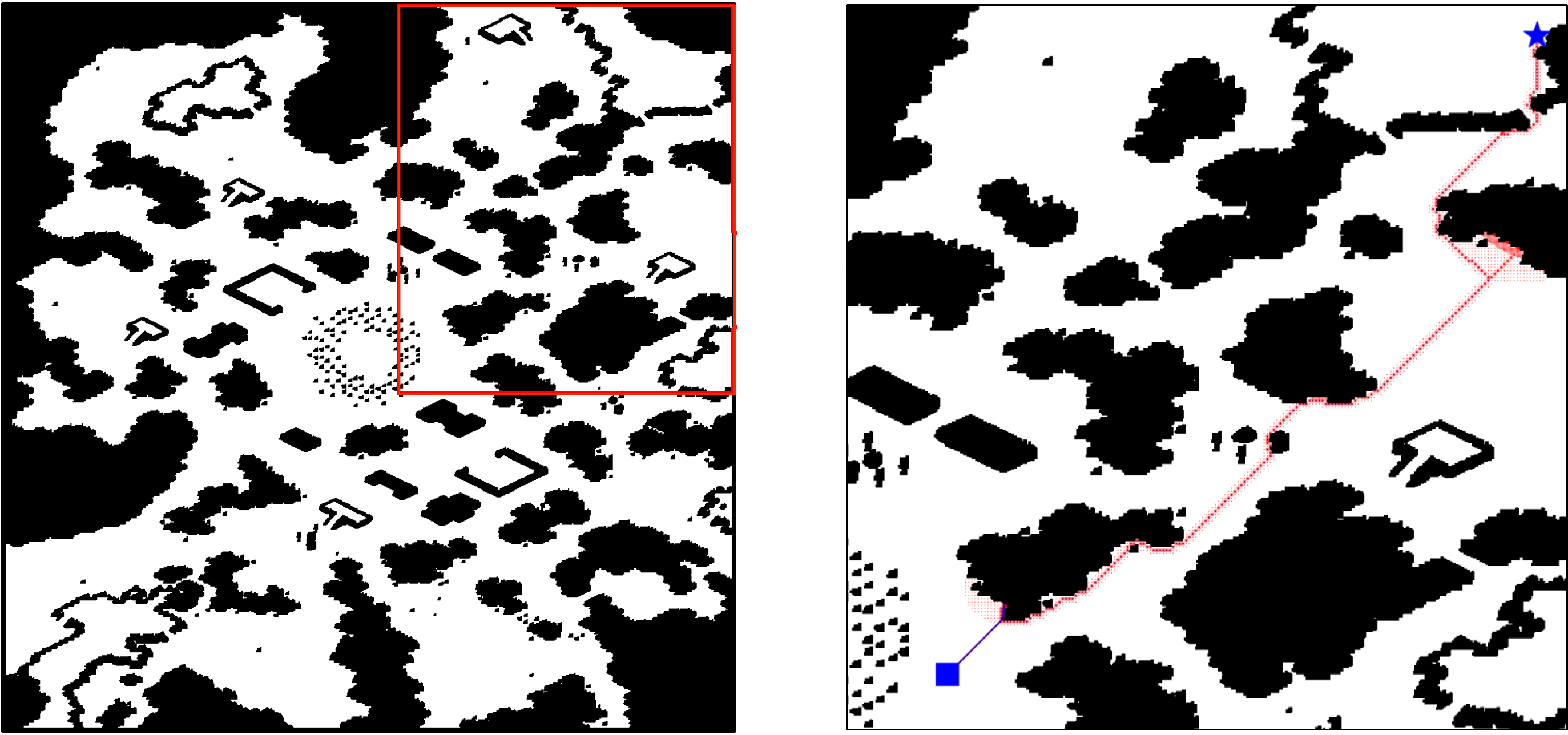}
    \caption{A 2D solution example. The planner is planning from start~(square) to goal~(star). 
    The red dots are expanded states and the blue line is the solution returned by planner.}\label{fig:2dsolexp}
\end{figure}
\subsubsection{Domain:}
We used two different maps discretized into~$10,000\times10,000$ cells as the highest resolution discretization.
Additionally, we have middle and low resolutions whose cells are 7 and 21 times the size of highest resolution cells respectively.
The benchmark maps are from \textit{Moving AI Lab}~\cite{sturtevant2012benchmarks} \textit{Starcraft} category.
For each map, we have~$100$ randomly generated start and goal pairs.
We compare our algorithm with four baselines, three of which search over implicit graph. These are WA* with Multiple Resolutions (WA-MR), WA* with highest resolution (WA-High) and with lowest resolution (WA-Low). WA-MR's action space uses the union of all the resolution spaces in a single queue.
The fourth baseline searches over a pre-constructed explicit graph that is the quad-tree search method~\cite{DBLP:conf/icra/GarciaKB14} (QDTree).
In quad-tree experiments, to book-keep neighbors of a grid, we followed the methods suggested in~\cite{DBLP:journals/cacm/LiL87,DBLP:journals/cacm/LiL87a}.
For our algorithm, we set the~$\omega_1$ and~$\omega_2$ values both to~$3.0$.
For other search-based algorithms, we set the weight to~$3.0$ as well, which would enforce the same suboptimality bounds for all the algorithms.

\subsubsection{Results and Analysis:}
The results of 2D planning are presented in Fig.~\ref{fig:rlt2d} and Table.~\ref{table:rlt2d}.
A test map and a sample solution from MRA* is shown in Fig.~\ref{fig:2dsolexp}.
In the top-right region of the right figure, we can see that MRA* sparsely searched the local minimum region and exited swiftly.
This is consistent with the behaviour that we described in Fig.~\ref{fig:alg_example}.

Our algorithm outperforms WA-MR and WA-High in speed and number of expansions as shown in Fig.~\ref{fig:rlt2d}. The speedup comes from the fact that WA-MR performs a full expansion of a every state which is expensive whereas MRA* only uses partial expansions.
WA-High searches only in the highest resolution which is also expensive, MRA* on the other hand leverages the low resolution space to quickly escape local minima and uses the high resolution space to plan through narrow passages.
WA-Low is faster than MRA* since it only searches in the lowest resolution space, but it also makes it incomplete with respect to the high resolution space. This is verified by the lowest success rate in Table.~\ref{table:rlt2d}.
QDTree is faster compared to MRA* because the quad-tree map discretization is done in such a way that large open spaces are not further discretized into smaller units, this helps to keep the size of state space small. However the graph construction step is computationally expensive and had an average pre-computation time of ~$36$ seconds for the two maps.
The quality of solutions as indicated by the average solution costs in Table.~\ref{table:rlt2d} for each algorithm is comparable except QDTree which relatively shows higher costs.
This is because QDTree has very coarse discretization in free spaces.
\subsection{3D Space Planning Results}
\subsubsection{Domain:}
For 3D also we used two maps, one of them is shown in Fig.~\ref{fig:map_3d}. 
The other map contains outdoor scenes such as mountains and buildings etc.
In the highest resolution, the maps are discretized to a grid of size~$1000\times1000\times400$ cells.
Similar to~2D spaces, we have middle and low resolutions that are~$9$ and~$27$ times the size of the highest resolution respectively.
There are~$50$ trails in total where start and goal pairs for each trial are randomly assigned.
For 3D experiments we only compared with the baselines which search on implicit graphs i.e. WA-MR, WA-High and WA-Low as the overhead of constructing the explicit abstraction for this domain is very high.
In our algorithm, we set the~$\omega_1$ and~$\omega_2$ value both to~$3.0$.
For other search based algorithms, we set the weights to~$3.0$ as well.
\begin{figure}[tb]
    \centering
    \fbox{\includegraphics[width=.95\columnwidth]{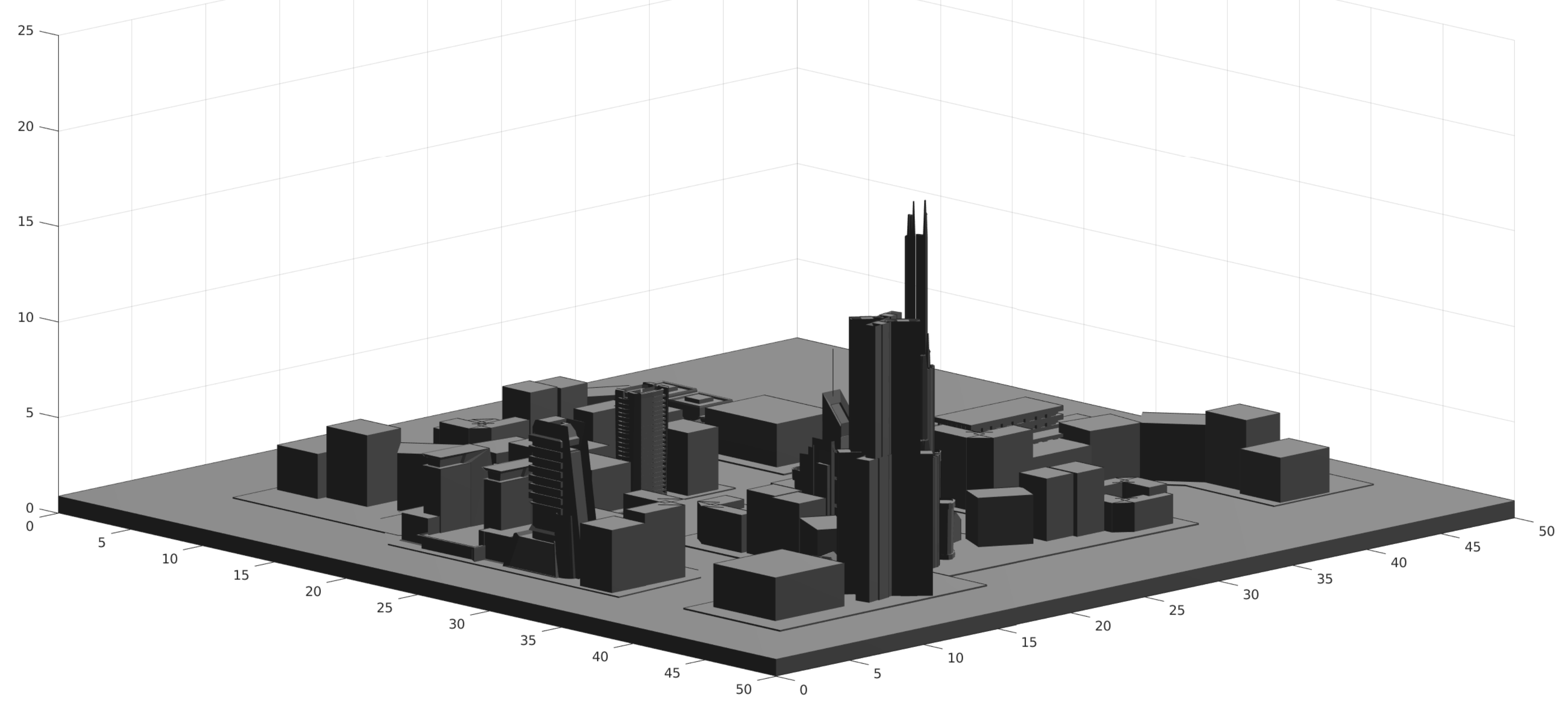}}
    \caption{A mesh model of city used as a planning scene for 3D planning.}\label{fig:map_3d}
\end{figure}
\subsubsection{Results and Analysis:}
The results for scene Fig.~\ref{fig:map_3d} are presented in Fig.~\ref{fig:rlt3d}.
With the same branching factor, WA* in coarse resolution space is significantly faster.
As mentioned earlier, the low resolution implementation is incomplete and the suboptimality bounds are also weaker, which results in lower success rate and poor quality solutions.
Regarding planning times, MRA* is the fastest as it leverages the different resolution spaces intelligently to quickly find solutions.

For WA-MR, as it performs full state expansions the branching factor becomes very large in 3D i.e. 78, which deteriorates it's performance (see Table.~\ref{table:rlt3d}).
In terms of solution cost, MRA* generates solutions slightly worse than WA-MR and WA-High, yet still bounded by the same suboptimality bound.
\begin{figure*}[bth]
    \centering
    \begin{subfigure}{0.32\textwidth}
        \centering
        \includegraphics[width=.95\columnwidth]{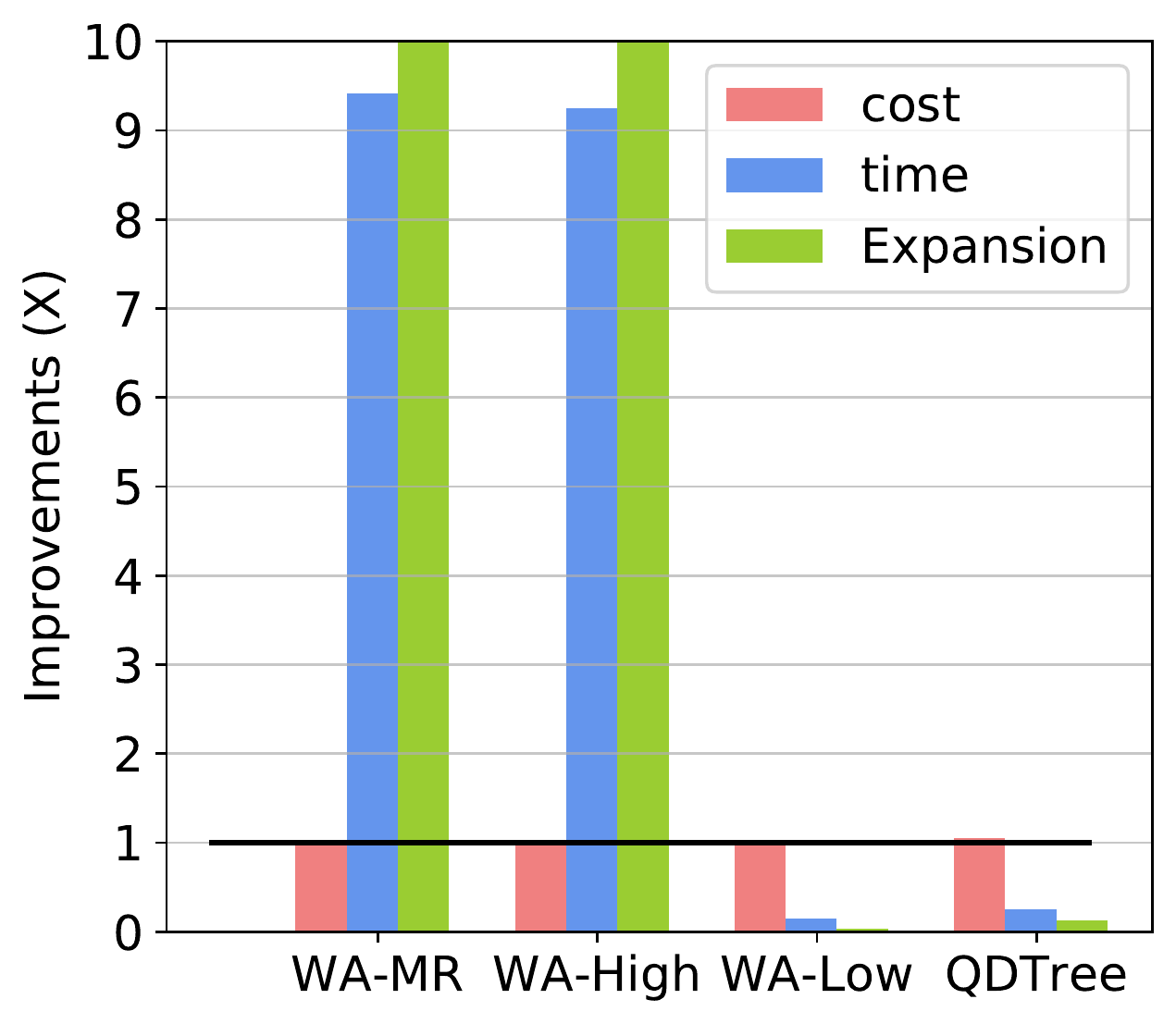}
        \centering
        \caption{The results of 2D planning.}\label{fig:rlt2d}
    \end{subfigure}
    \begin{subfigure}{0.32\textwidth}
        \centering
        \includegraphics[width=.95\columnwidth]{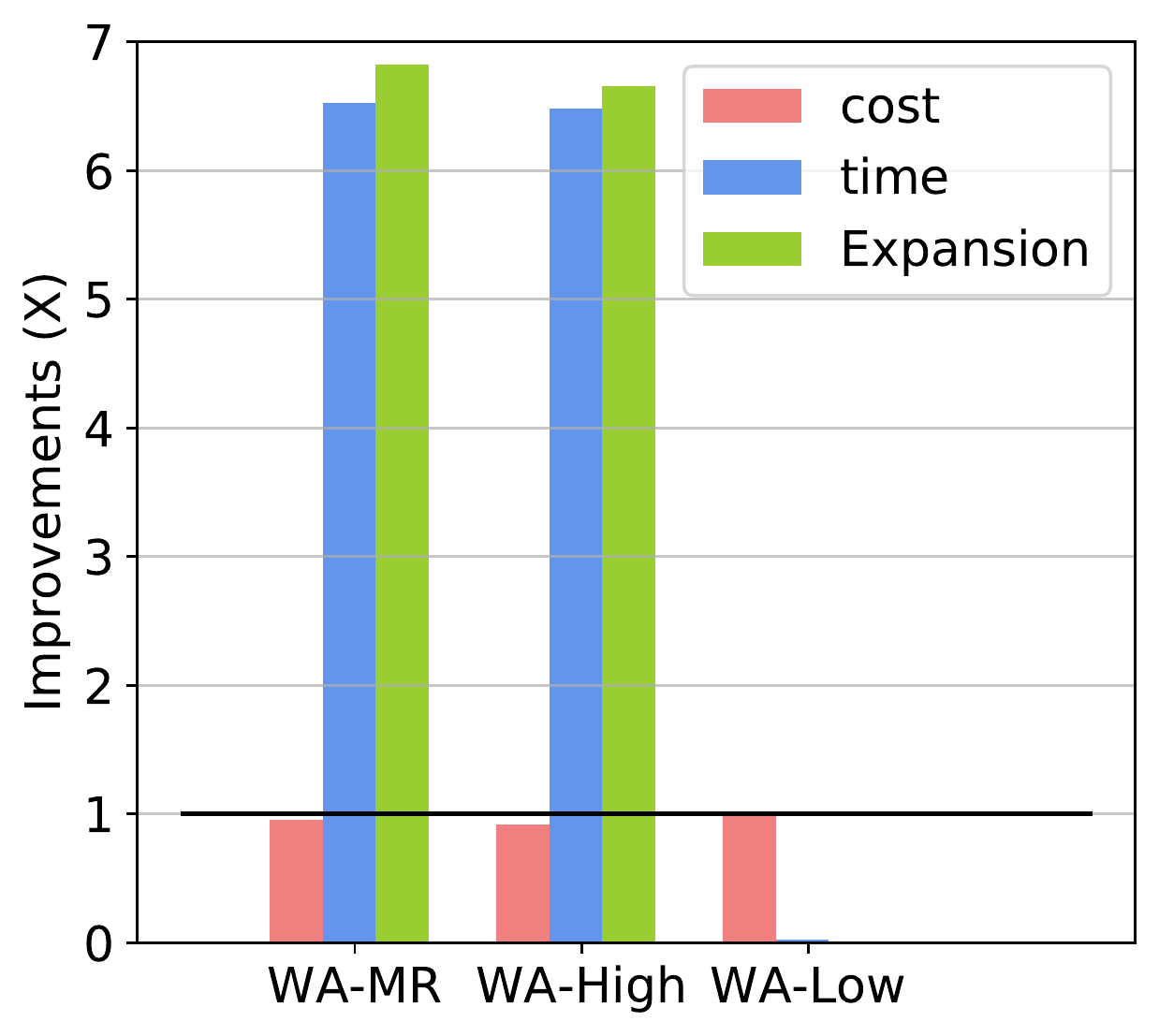}\caption{3D planning results in scene Fig.~\ref{fig:map_3d}}\label{fig:rlt3d}
        \end{subfigure}
    \begin{subfigure}{0.32\textwidth}
        \centering
        \includegraphics[width=.95\columnwidth]{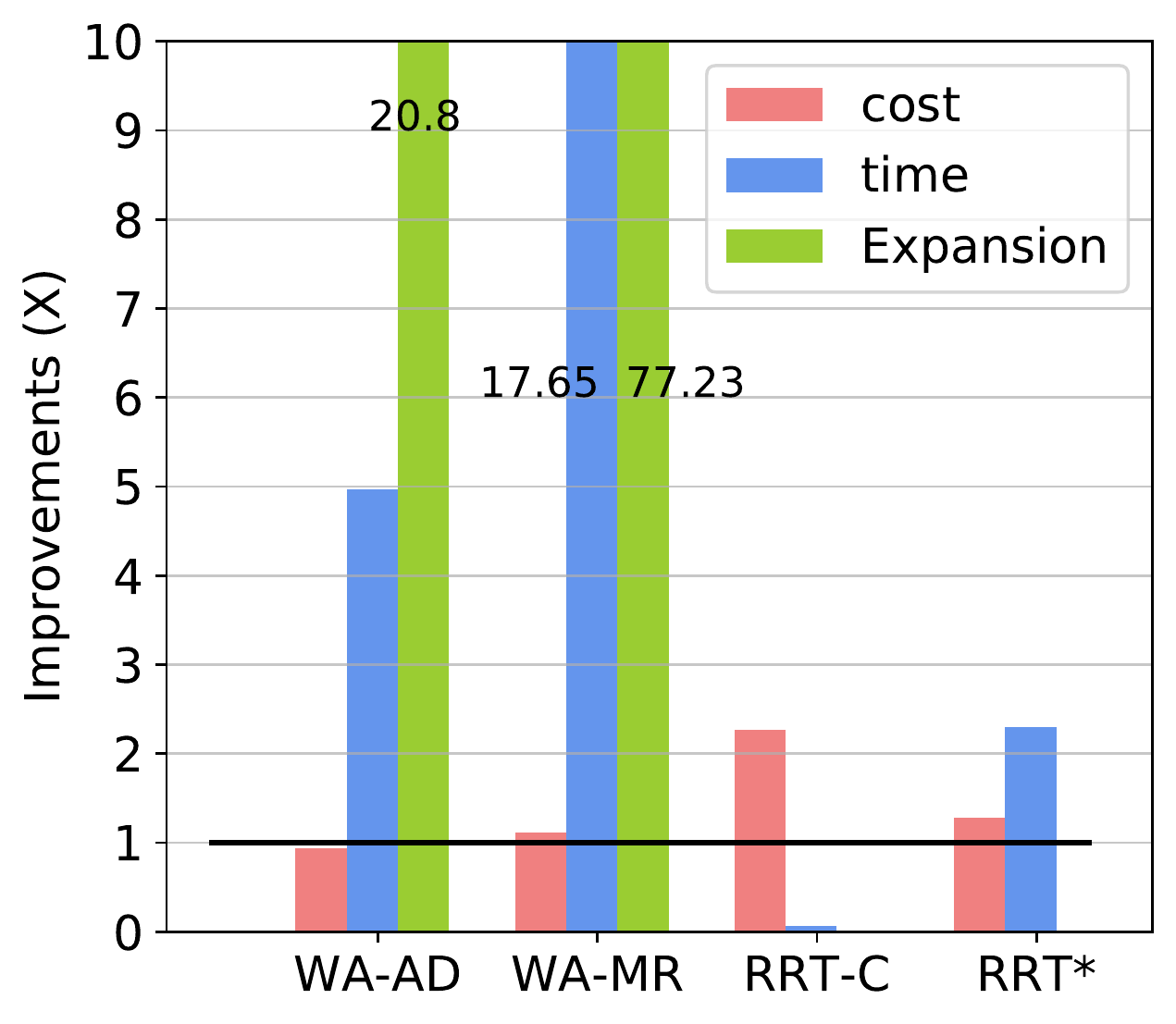}
        \caption{7D planning in scene Industrial.}\label{fig:rlt7d}
    \end{subfigure}
    \caption{Improvements of MRA* over baseline algorithms}\label{fig:all_results}
\end{figure*}

\subsection{7D Space Planning Results}
For 7D domain implementation we used an adaptation of {SMPL}\footnote{\url{https://github.com/aurone/smpl}}.
\subsubsection{Domain:}
We used PR2 robot's 7DoF arm for this domain.
We ran the experiments on four different benchmark scenarios~\cite{DBLP:journals/ijrr/CohenCL14} as in Fig.~\ref{fig:7dexample}.
The start and goal pairs were randomly generated
for~$70$ trails for each scene.
We used RRT-Connect (RRT-C) and RRT* as the sampling-based planning baselines.
In addition, we tested with WA-MR and WA* with adaptive dimensionality search~\cite{DBLP:conf/aips/GochevSL13} (WA-AD) as search-based planning baselines.
The implementations of sampling-based approaches are used from Open Motion Planning Library~(OMPL)~\cite{sucan2012the-open-motion-planning-library}. For RRT* we report the results for the first solution found.
For search-based algorithms, we set the weights for WA* search to be~$25$.
In our algorithm, we set the~$\omega_1$ and~$\omega_2$ value to~$20$ and~$25$ respectively.
\begin{figure*}[hbt]
    \centering
    \begin{subfigure}[b]{0.50\columnwidth}
        \centering
        \includegraphics[width=0.95\columnwidth]{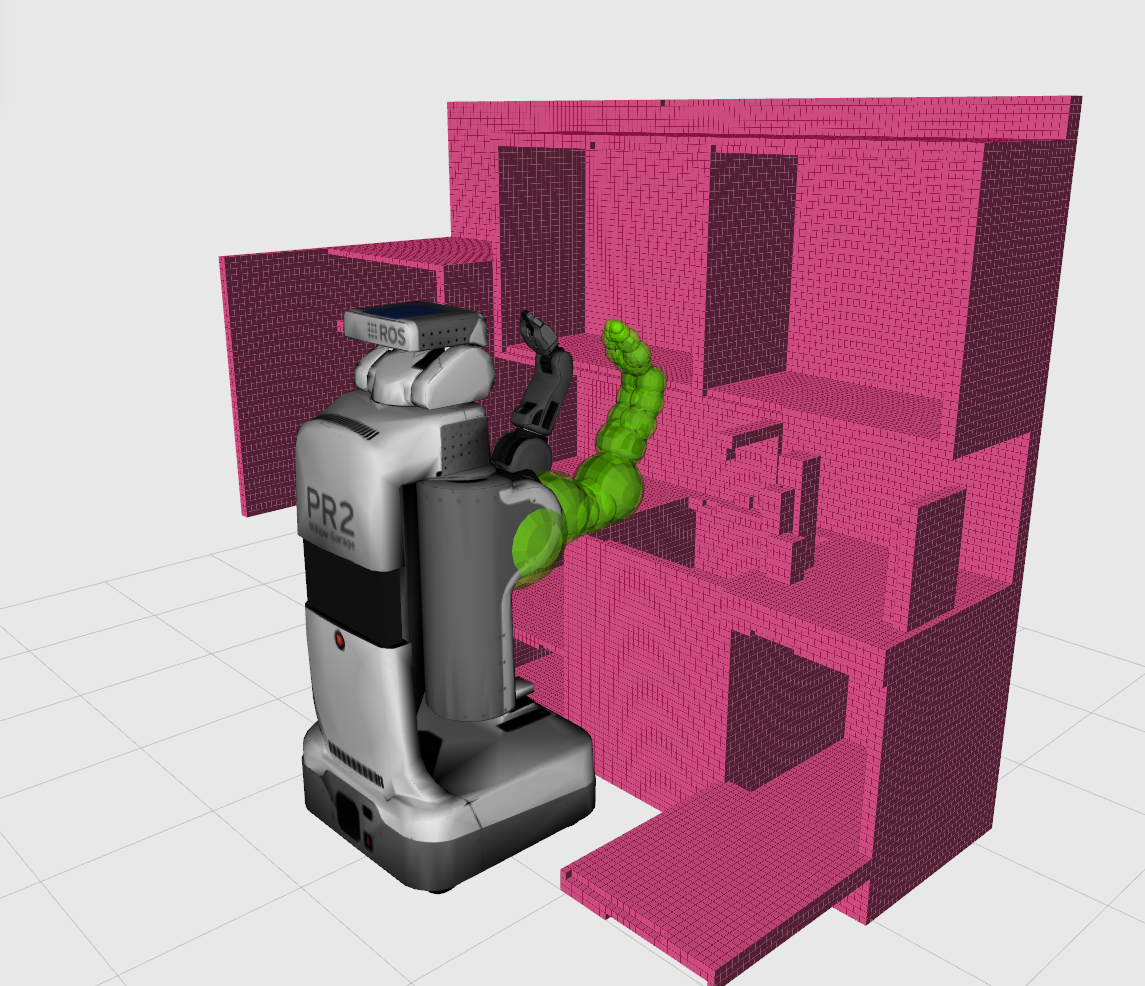}\caption{Kitchen}\label{fig:kitchen}
    \end{subfigure}
    \begin{subfigure}[b]{0.50\columnwidth}
        \centering
        \includegraphics[width=0.95\columnwidth]{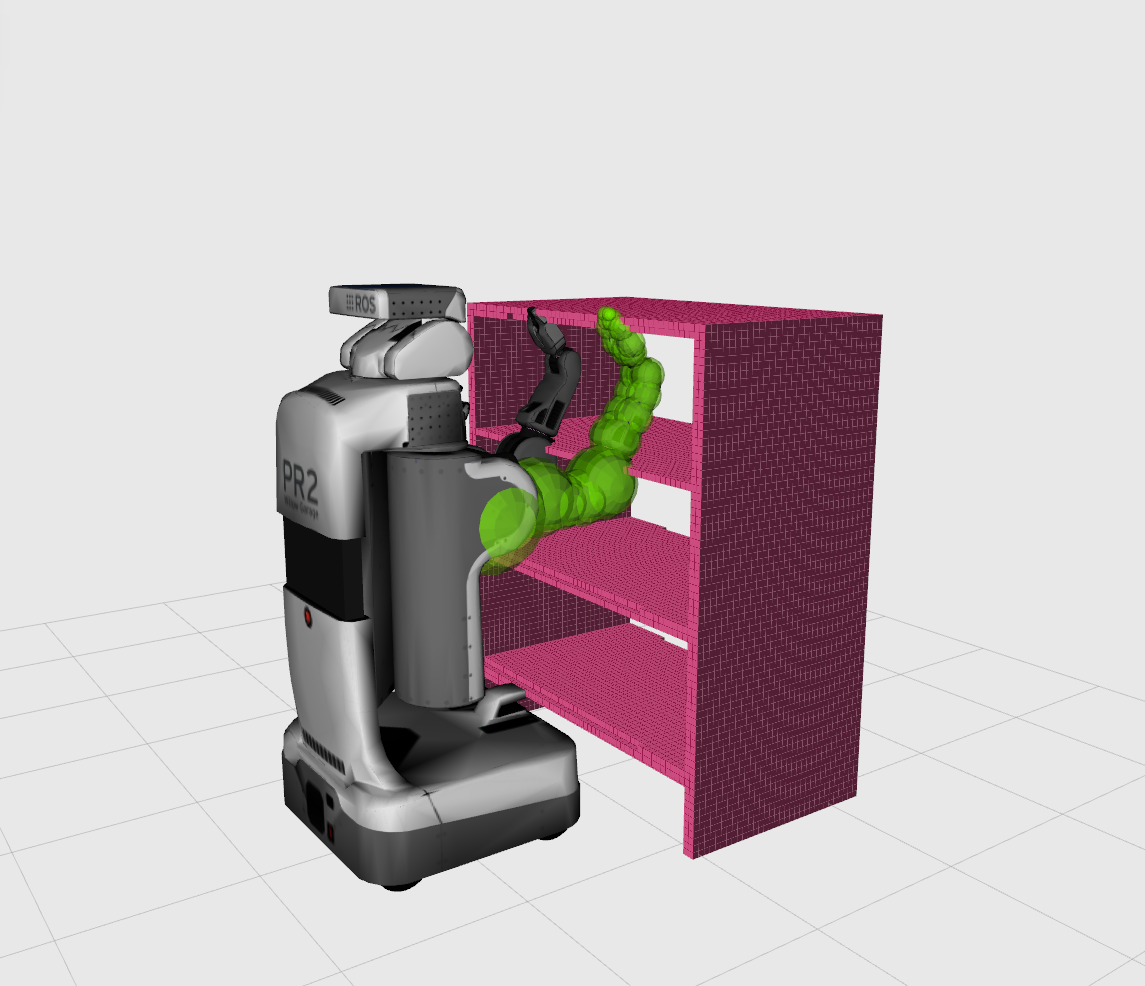}\caption{Bookshelf}\label{fig:bookshelf}
    \end{subfigure}
    \begin{subfigure}[b]{0.50\columnwidth}
        \centering
        \includegraphics[width=0.95\columnwidth]{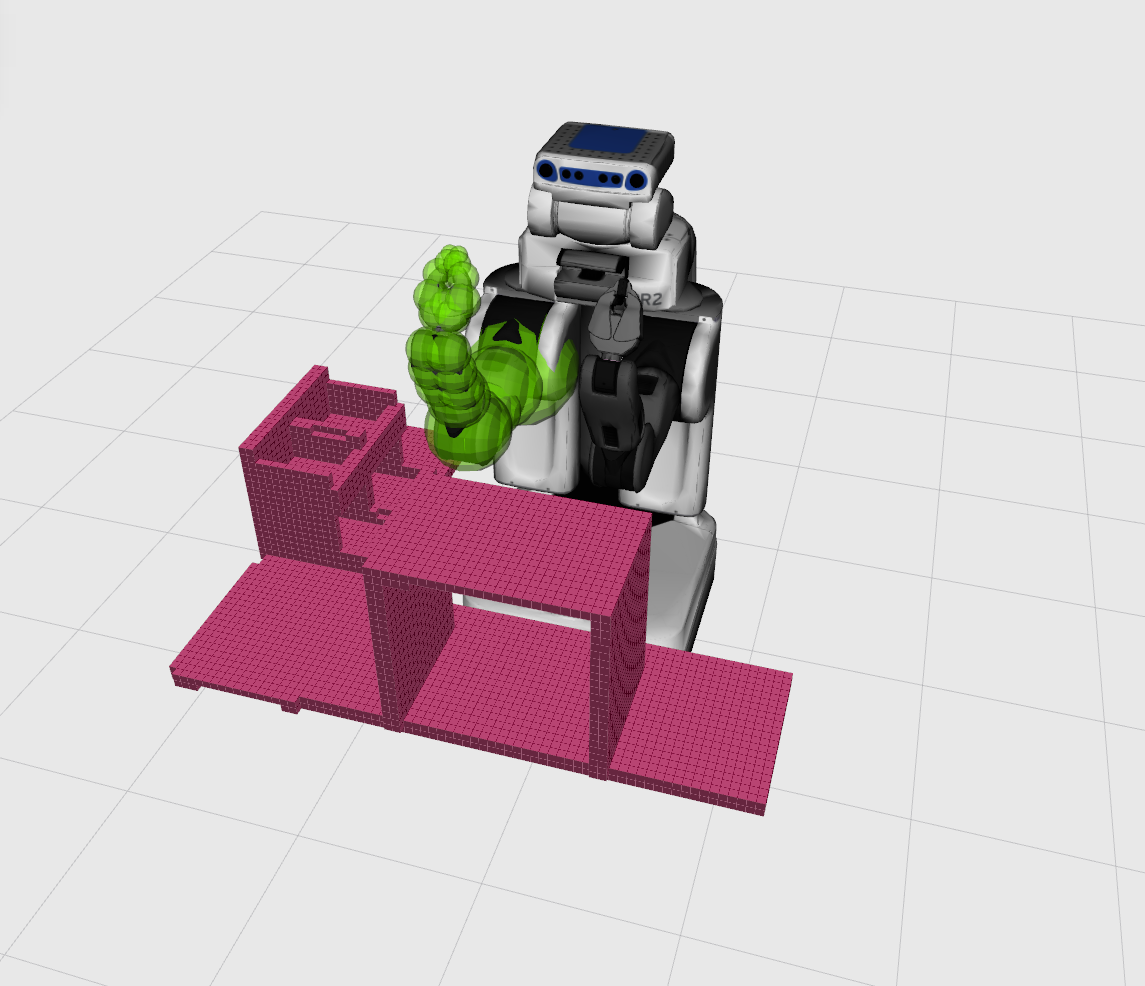}\caption{Industrial}\label{fig:industrial}
    \end{subfigure}
    \begin{subfigure}[b]{0.50\columnwidth}
        \centering
        \includegraphics[width=0.95\columnwidth]{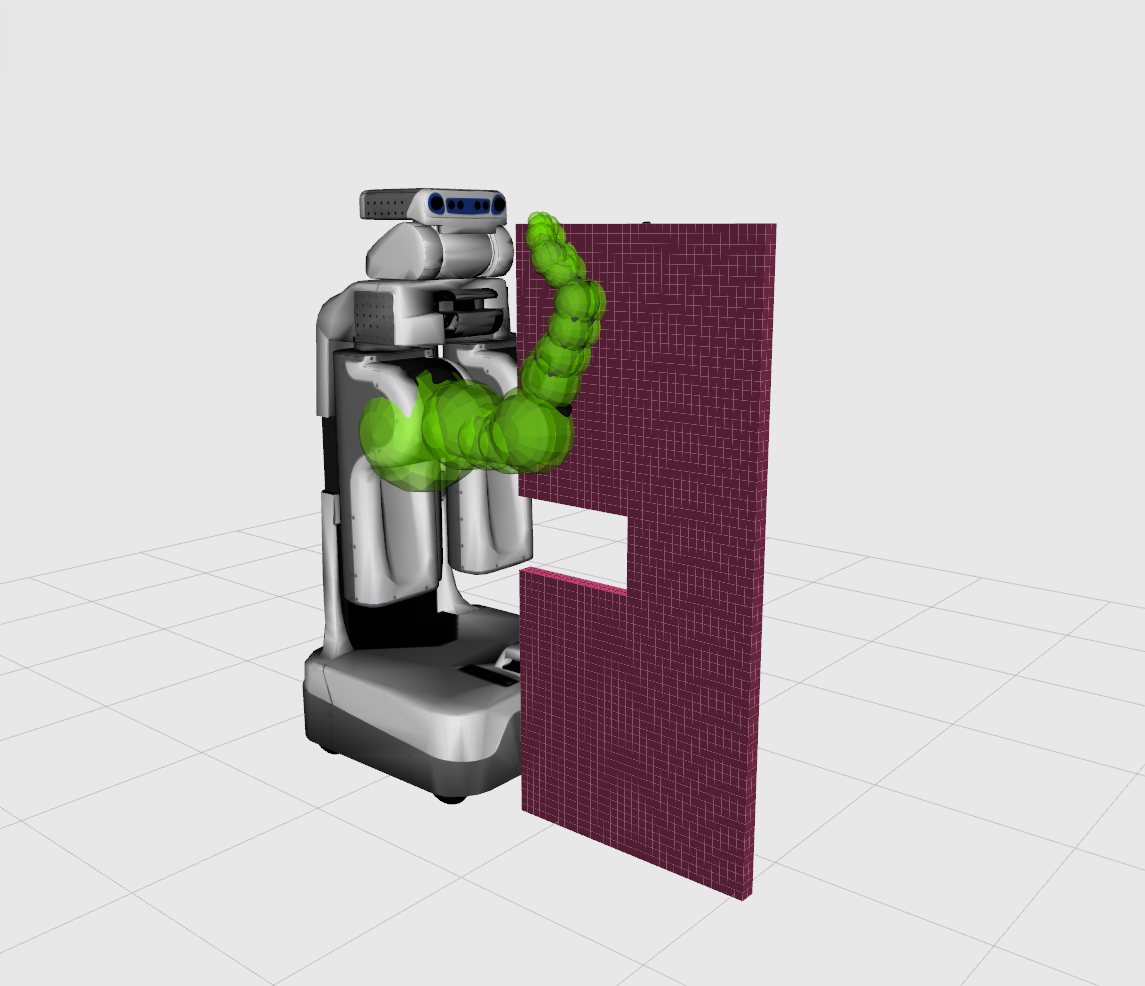}\caption{Narrow Passage}\label{fig:npassage}
    \end{subfigure}
    \caption{The planning scenes of single-arm manipulation problem.}
\label{fig:7dexample}
\end{figure*}
\begin{table*}[bth]
    \caption{7D planning results on 4 scenes.}\label{table:rlt7d}
    \begin{subtable}{\textwidth}
        \resizebox{\columnwidth}{!}{%
            \centering
             \begin{tabular}{c|c|c|c|c|c||c|c|c|c|c}
                    \hline \hline
                     & \multicolumn{5}{|c||}{Kitchen} & \multicolumn{5}{|c}{Bookshelf}\\
                    \hline
                    Algorithm  & MRA*   & WA-AD   & WA-MR   & RRT-C & RRT* &  MRA*   & WA-AD   & WA-MR   & RRT-C & RRT*\\
                    \hline 
                    Success Rate (\%)  & 95.24  & 49.21  & 46.031     & \textbf{95.83}          & 75.00 & \textbf{100}  & 57.38   & 47.54           & 89.79  & 44.00\\
                    \hline                           
                    Mean Time (s)      & 3.44            & 12.55  & 9.19   & \textbf{0.006} & 1.04 & 2.83  & 8.44    & 11.62  & \textbf{0.13} & 9.74\\
                    \hline                           
                    Mean Cost ($rad$)  & 7.57  & 6.22   & \textbf{5.96}   & 15.49    & 8.16 & 11.21     & 9.74    & \textbf{10.38}  & 28.54   & 15.70\\
                    \hline                           
                    Processed Mean Cost ($rad$)  & 6.96  & 5.22   & \textbf{5.26}   & 8.9   & 7.25 & 10.77  & \textbf{9.13}    & 9.15   & 16.93   & 13.59\\
                    \hline \hline
            \end{tabular}}
    \end{subtable}%
    \hfill
    \begin{subtable}{\textwidth}
        \centering
        \resizebox{\columnwidth}{!}{%
             \begin{tabular}{c|c|c|c|c|c||c|c|c|c|c}
                    \hline \hline
                     & \multicolumn{5}{|c||}{Industrial} & \multicolumn{5}{|c}{Narrow Passage}\\
                    \hline
                    Algorithm  & MRA*   & WA-AD   & WA-MR   & RRT-C & RRT* &  MRA*   & WA-AD   & WA-MR   & RRT-C & RRT*\\
                    \hline 
                    Success Rate (\%)  & \textbf{96.92}  & 72.31  & 15.38  & 89.83  & 62.07 & \textbf{100}    & 50.00   & 40.91 & 96.22  & 67.27\\
                    \hline                           
                    Mean Time (s)      & 3.13   & 7.61   & 15.48  & \textbf{0.29} & 9.84 & 4.30   & 7.98    & 15.21  & \textbf{0.05} & 4.70\\
                    \hline                           
                    Mean Cost ($rad$)  & 13.12 & 12.77  & \textbf{11.10}  & 29.26   & 16.38 & 11.92  & 10.71   & \textbf{10.12} & 20.90   & 14.39\\
                    \hline                           
                    Processed Mean Cost ($rad$)  & 12.67 & 11.20  & \textbf{10.53}  & 16.29   & 13.77 & 11.59  & 10.60   & \textbf{9.91}  & 12.42   & 12.20\\
                    \hline \hline
            \end{tabular}}
    \end{subtable}
\end{table*}
\subsubsection{Motion Primitives:}
A base set of~$14$ motion primitives are provided and categorized into classes with low, middle and high resolutions:~$M_{\rm low}$,~$M_{\rm middle}$,~$M_{\rm high}$.
Each motion primitive changes the position of one joint in both directions by an amount corresponding to the resolution.
In~$M_{\rm low}$,~$M_{\rm middle}$ and~$M_{\rm high}$ each action corresponds to a joint angle change of~$27^\circ$,~$9^\circ$ and~$3^\circ$ respectively.
In addition to the static motion primitives, adaptive actions are generated online via inverse kinematics computation~\cite{DBLP:conf/icra/CohenSCL11} to \textit{snap} end-effector to the goal pose when the expanded state is within a small threshold distance to the goal position.

\subsubsection{Results and Analysis:}
We show the experimental results for the \textit{Industrial} scene (Fig.~\ref{fig:industrial}) presented in Fig.~\ref{fig:rlt7d}. The statistics for the other scenes are very similar and are omitted.
In terms of planning times, MRA* outperforms all the baselines except RRT-Connect. MRA* shows over an order of magnitude improvements over WA-AD and WA-MR in planning times and number of expansions, indicating that the performance gains are higher in higher dimension domains.
With respect to solution cost, MRA* performs no worse than any other algorithm on common succeeded trials.

From the results documented in Table.~\ref{table:rlt7d},
MRA* has consistently high success rates across all the scenes.
Although MRA* is slower than RRT-Connect in terms of solution costs, MRA* (and other search-based baselines) consistently show better solution qualities then RRT-Connect and even RRT*.
While WA-MR performs worst in terms of planning time and success rate, it consistently provides the best quality solutions, which could be explained by the fact that WA-MR searches in the graph which is the union of all resolution spaces, and has stricter suboptimality bounds.

\section{Discussion}
In this section we discuss the choice of algorithm parameters and the selection of resolutions for MRA* searches. We analysed the effect of varying the parameters,~$\omega_1$ and~$\omega_2$, on the performance of MRA*.
We fixed~$\omega_1=3.0$ and varied~$\omega_2$ and vice versa linearly to analyse the effects of each parameter independently. The results for the 2D domain are shown in Fig.~\ref{fig:weight}.
Increasing~$\omega_2$ speeds up the search as it allows more expansions from inadmissible (courser resolution) searches.
Increasing~$\omega_1$, first speeds up the search because it makes the inadmissible searches more greedy. However, after~$\omega_1=2$, the search slows down as MRA* starts expanding more states from the anchor search. 
\begin{figure}[bt]
    \centering
    \includegraphics[width=0.95\columnwidth]{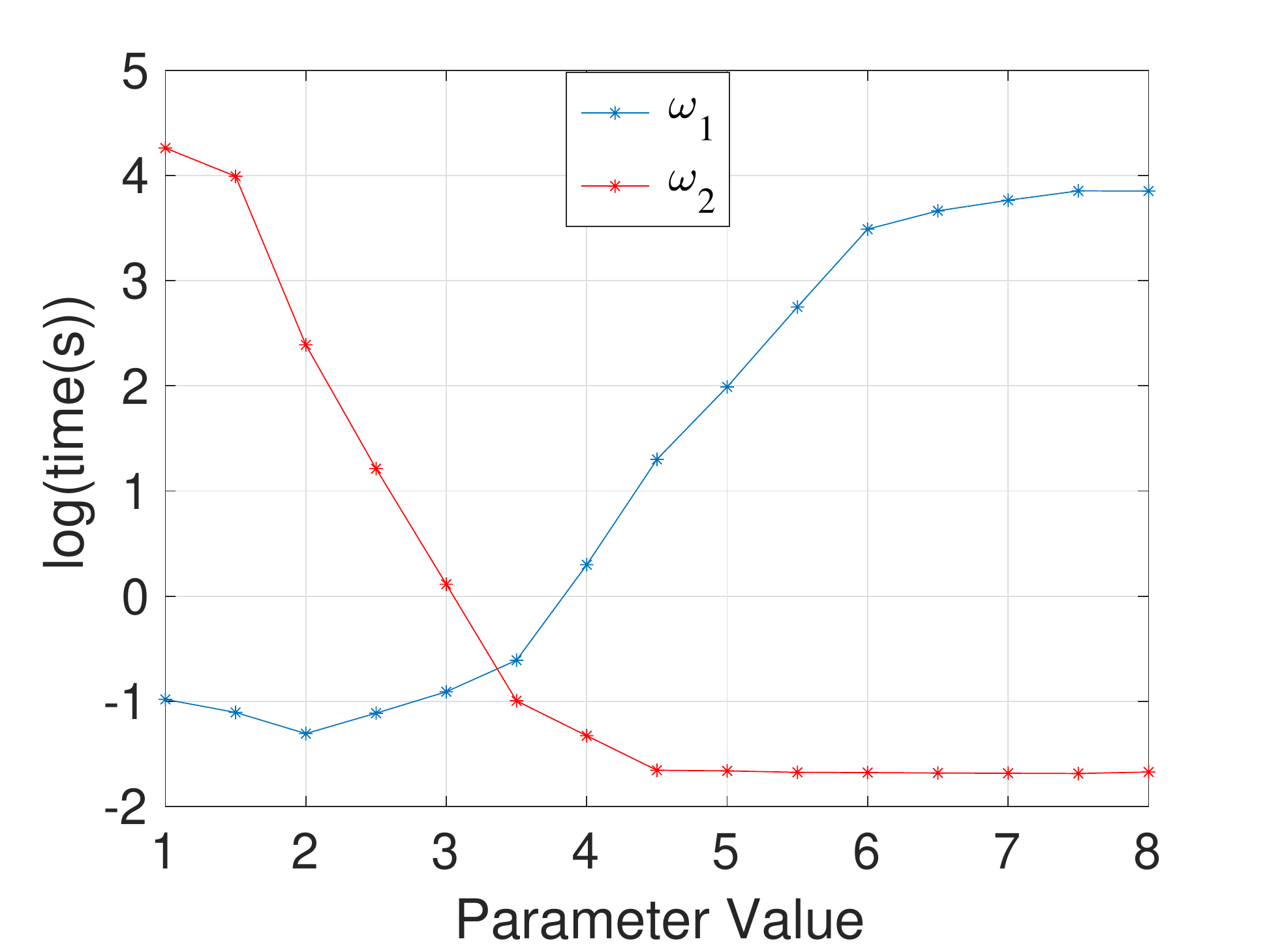}
    \caption{ Parameters~($\omega_1$/$\omega_2$) vs. Planning time($s$) in logarithm.
    Each parameter combination is tested with 3 different start and goal pairs. }\label{fig:weight}
\end{figure}

Besides the algorithm parameters, the choice of resolutions also significantly affects algorithm's performance. While the choice largely depends on the domain, resolutions should be selected such that the spaces are considerably overlapped so that more sharing is facilitated.
Our resolution selection criterion ensures that the centers of a lower resolution cells always coincide with the centers of higher resolution cells. As a consequence, the states in the lower resolution spaces will always be shared with the higher resolution spaces.
We do not claim that it is an optimal selection scheme and there definitely is more room for investigation.
\section{Conclusion and Future Work}
We presented a heuristic search-based algorithm that utilises multiple search spaces implicitly constructed with different resolutions and shares information between them.
We show that MRA* is resolution complete in the union resolution space and the solution cost returned by MRA* is bounded sub-optimal with respect to the optimal solution cost in the anchor resolution space.
We show that MRA* presents performance improvements over the baselines on large 2D, 3D domains and high-dimensional motion planning problems, most importantly in terms of success rates which are consistently high across all the domains and experiments.
While the results are promising, we believe that there is scope for further improvements.
Possible future directions can be 1) using multiple heuristics within the different resolutions searches to speed up the search
2) adding dynamic motions primitives for efficient sharing between the different spaces
3) using a large ensemble of resolution spaces and optimizing for the scheduling policy
and 4) using the multi-resolution framework for other bounded suboptimal search algorithms such as Optimistic Search~\cite{DBLP:conf/aips/ThayerR08} or search with different priority functions~\cite{DBLP:conf/ijcai/ChenS19}.

\section{Acknowledgements}
This work was in part supported by ONR grant N00014-18-1-2775.

\bibliographystyle{aaai}
\bibliography{reference}

\begin{thebibliography}{}

\bibitem[\protect\citeauthoryear{Aine \bgroup et al\mbox.\egroup
  }{2016}]{doi:10.1177/0278364915594029}
Aine, S.; Swaminathan, S.; Narayanan, V.; Hwang, V.; and Likhachev, M.
\newblock 2016.
\newblock {Multi-Heuristic A*}.
\newblock {\em The International Journal of Robotics Research~({IJRR})}
  35(1-3):224--243.

\bibitem[\protect\citeauthoryear{Bellman}{1957}]{Bellman:1957}
Bellman, R.
\newblock 1957.
\newblock {\em Dynamic Programming}.
\newblock Princeton University Press.

\bibitem[\protect\citeauthoryear{Brock and
  Kavraki}{2001}]{DBLP:conf/icra/BrockK01}
Brock, O., and Kavraki, L.~E.
\newblock 2001.
\newblock Decomposition-based motion planning: {A} framework for real-time
  motion planning in high-dimensional spaces.
\newblock In {\em {IEEE} International Conference on Robotics and
  Automation~({ICRA})},  1469--1474.

\bibitem[\protect\citeauthoryear{Cohen \bgroup et al\mbox.\egroup
  }{2011}]{DBLP:conf/icra/CohenSCL11}
Cohen, B.~J.; Subramania, G.; Chitta, S.; and Likhachev, M.
\newblock 2011.
\newblock Planning for manipulation with adaptive motion primitives.
\newblock In {\em {IEEE} International Conference on Robotics and
  Automation~({ICRA})},  5478--5485.

\bibitem[\protect\citeauthoryear{Cohen, Chitta, and
  Likhachev}{2010}]{DBLP:conf/icra/CohenCL10}
Cohen, B.~J.; Chitta, S.; and Likhachev, M.
\newblock 2010.
\newblock Search-based planning for manipulation with mootion primitives.
\newblock In {\em {IEEE} International Conference on Robotics and
  Automation~({ICRA})},  2902--2908.

\bibitem[\protect\citeauthoryear{Cohen, Chitta, and
  Likhachev}{2014}]{DBLP:journals/ijrr/CohenCL14}
Cohen, B.~J.; Chitta, S.; and Likhachev, M.
\newblock 2014.
\newblock Single- and dual-arm motion planning with heuristic search.
\newblock {\em The International Journal of Robotics Research~({IJRR})}
  33(2):305--320.

\bibitem[\protect\citeauthoryear{Elbanhawi and
  Simic}{2014}]{DBLP:journals/access/ElbanhawiS14}
Elbanhawi, M., and Simic, M.
\newblock 2014.
\newblock Sampling-based robot motion planning: {A} review.
\newblock {\em {IEEE} Access} 2:56--77.

\bibitem[\protect\citeauthoryear{Gammell, Srinivasa, and
  Barfoot}{2014}]{DBLP:conf/iros/GammellSB14}
Gammell, J.~D.; Srinivasa, S.~S.; and Barfoot, T.~D.
\newblock 2014.
\newblock Informed rrt*: Optimal sampling-based path planning focused via
  direct sampling of an admissible ellipsoidal heuristic.
\newblock In {\em {IEEE/RSJ} International Conference on Intelligent Robots and
  Systems~({IROS})},  2997--3004.

\bibitem[\protect\citeauthoryear{Garcia, Kapadia, and
  Badler}{2014}]{DBLP:conf/icra/GarciaKB14}
Garcia, F.~M.; Kapadia, M.; and Badler, N.~I.
\newblock 2014.
\newblock Gpu-based dynamic search on adaptive resolution grids.
\newblock In {\em {IEEE} International Conference on Robotics and
  Automation~({ICRA})},  1631--1638.

\bibitem[\protect\citeauthoryear{Gupta, Granmo, and
  Agrawala}{2011}]{DBLP:conf/icmla/GuptaGA11}
Gupta, N.; Granmo, O.; and Agrawala, A.~K.
\newblock 2011.
\newblock Thompson sampling for dynamic multi-armed bandits.
\newblock In {\em International Conference on Machine Learning and Applications
  and Workshops~({ICMLA})},  484--489.

\bibitem[\protect\citeauthoryear{Islam \bgroup et al\mbox.\egroup
  }{2012}]{islam2012rrt}
Islam, F.; Nasir, J.; Malik, U.; Ayaz, Y.; and Hasan, O.
\newblock 2012.
\newblock Rrt*-smart: Rapid convergence implementation of rrt* towards optimal
  solution.
\newblock In {\em IEEE International Conference on Mechatronics and
  Automation},  1651--1656.
\newblock IEEE.

\bibitem[\protect\citeauthoryear{Janson \bgroup et al\mbox.\egroup
  }{2015}]{DBLP:journals/ijrr/JansonSCP15}
Janson, L.; Schmerling, E.; Clark, A.~A.; and Pavone, M.
\newblock 2015.
\newblock Fast marching tree: {A} fast marching sampling-based method for
  optimal motion planning in many dimensions.
\newblock {\em The International Journal of Robotics Research~({IJRR})}
  34(7):883--921.

\bibitem[\protect\citeauthoryear{Jr. and
  LaValle}{2000}]{DBLP:conf/icra/KuffnerL00}
Jr., J. J.~K., and LaValle, S.~M.
\newblock 2000.
\newblock Rrt-connect: An efficient approach to single-query path planning.
\newblock In {\em {IEEE} International Conference on Robotics and
  Automation~({ICRA})},  995--1001.

\bibitem[\protect\citeauthoryear{Kalin~Gochev and
  Likhachev}{2013}]{DBLP:conf/aips/GochevSL13}
Kalin~Gochev, A.~S., and Likhachev, M.
\newblock 2013.
\newblock Incremental planning with adaptive dimensionality.
\newblock In {\em International Conference on Automated Planning and
  Scheduling~({ICAPS})}.

\bibitem[\protect\citeauthoryear{Karaman and
  Frazzoli}{2011}]{karaman2011sampling}
Karaman, S., and Frazzoli, E.
\newblock 2011.
\newblock Sampling-based algorithms for optimal motion planning.
\newblock {\em The International Journal of Robotics Research~({IJRR})}
  846--894.

\bibitem[\protect\citeauthoryear{LaValle}{2006}]{lavalle2006planning}
LaValle, S.~M.
\newblock 2006.
\newblock {\em Planning algorithms}.
\newblock Cambridge university press.

\bibitem[\protect\citeauthoryear{Li and
  Loew}{1987a}]{DBLP:journals/cacm/LiL87a}
Li, S., and Loew, M.~H.
\newblock 1987a.
\newblock Adjacency detection using quadcodes.
\newblock {\em Communications of the ACM} 30(7):627--631.

\bibitem[\protect\citeauthoryear{Li and Loew}{1987b}]{DBLP:journals/cacm/LiL87}
Li, S., and Loew, M.~H.
\newblock 1987b.
\newblock The quadcode and its arithmetic.
\newblock {\em Communications of the ACM} 30(7):621--626.

\bibitem[\protect\citeauthoryear{Likhachev and
  Ferguson}{2009}]{DBLP:journals/ijrr/LikhachevF09}
Likhachev, M., and Ferguson, D.
\newblock 2009.
\newblock Planning long dynamically feasible maneuvers for autonomous vehicles.
\newblock {\em The International Journal of Robotics Research~({IJRR})}
  28(8):933--945.

\bibitem[\protect\citeauthoryear{Moore and
  Atkeson}{1995}]{DBLP:journals/ml/MooreA95}
Moore, A.~W., and Atkeson, C.~G.
\newblock 1995.
\newblock The parti-game algorithm for variable resolution reinforcement
  learning in multidimensional state-spaces.
\newblock {\em Machine Learning} 21(3):199--233.

\bibitem[\protect\citeauthoryear{Pearl}{1984}]{pearl1984heuristics}
Pearl, J.
\newblock 1984.
\newblock {\em Heuristics: intelligent search strategies for computer problem
  solving}.
\newblock Addison-Wesley Pub. Co., Inc., Reading, MA.

\bibitem[\protect\citeauthoryear{Petrovic}{2018}]{DBLP:journals/corr/abs-1806-07457}
Petrovic, L.
\newblock 2018.
\newblock Motion planning in high-dimensional spaces.
\newblock {\em Computing Research Repository~({CoRR})} abs/1806.07457.

\bibitem[\protect\citeauthoryear{Phillips \bgroup et al\mbox.\egroup
  }{2015}]{DBLP:conf/ijcai/PhillipsNAL15}
Phillips, M.; Narayanan, V.; Aine, S.; and Likhachev, M.
\newblock 2015.
\newblock Efficient search with an ensemble of heuristics.
\newblock In {\em International Joint Conferences on Artificial
  Intelligence~({IJCAI})},  784--791.

\bibitem[\protect\citeauthoryear{Pivtoraiko and
  Kelly}{2005}]{DBLP:conf/iros/PivtoraikoK05}
Pivtoraiko, M., and Kelly, A.
\newblock 2005.
\newblock Generating near minimal spanning control sets for constrained motion
  planning in discrete state spaces.
\newblock In {\em {IEEE/RSJ} International Conference on Intelligent Robots and
  Systems~({IROS})},  3231--3237.

\bibitem[\protect\citeauthoryear{Pohl}{1973}]{pohl1973avoidance}
Pohl, I.
\newblock 1973.
\newblock The avoidance of~(relative) catastrophe, heuristic competence,
  genuine dynamic weighting and computational issues in heuristic problem
  solving.
\newblock In {\em International Joint Conferences on Artificial
  Intelligence~({IJCAI})},  12--17.

\bibitem[\protect\citeauthoryear{Sturtevant}{2012}]{sturtevant2012benchmarks}
Sturtevant, N.
\newblock 2012.
\newblock Benchmarks for grid-based pathfinding.
\newblock {\em Transactions on Computational Intelligence and AI in Games}
  4(2):144 -- 148.

\bibitem[\protect\citeauthoryear{{\c{S}}ucan, Moll, and
  Kavraki}{2012}]{sucan2012the-open-motion-planning-library}
{\c{S}}ucan, I.~A.; Moll, M.; and Kavraki, L.~E.
\newblock 2012.
\newblock The {O}pen {M}otion {P}lanning {L}ibrary.
\newblock {\em {IEEE} Robotics \& Automation Magazine} 19(4):72--82.
\newblock \url{http://ompl.kavrakilab.org}.

\bibitem[\protect\citeauthoryear{Vemula, M{\"{u}}lling, and
  Oh}{2016}]{DBLP:conf/socs/VemulaMO16}
Vemula, A.; M{\"{u}}lling, K.; and Oh, J.
\newblock 2016.
\newblock Path planning in dynamic environments with adaptive dimensionality.
\newblock In {\em Symposium on Combinatorial Search~(SoCS)},  107--116.

\bibitem[\protect\citeauthoryear{Yahja \bgroup et al\mbox.\egroup
  }{1998}]{DBLP:conf/icra/YahjaSSB98}
Yahja, A.; Stentz, A.; Singh, S.; and Brumitt, B.
\newblock 1998.
\newblock Framed-quadtree path planning for mobile robots operating in sparse
  environments.
\newblock In {\em {IEEE} International Conference on Robotics and
  Automation~({ICRA})},  650--655.

\end{thebibliography}

\end{document}